\documentclass[twoside,leqno,twocolumn]{article}  
\usepackage{ltexpprt_mod}

\usepackage{amsmath}
\usepackage{amssymb}

\newtheorem{definition}{Definition}[section]
\newtheorem{example}{Example}
\newtheorem{proper}{Property}[section]

\usepackage{thmtools}
\usepackage{thm-restate}

\usepackage{tikz}
\usetikzlibrary{shapes,backgrounds}
\usepackage{verbatim}
\usepackage{pgfplots}
\usepgflibrary{shapes.geometric}
\usetikzlibrary{calc}

\usepackage{sidecap}
\sidecaptionvpos{figure}{c}
\usepackage{subfigure}

\usetikzlibrary{arrows}
\usepackage{color}

\usepackage{booktabs}

\usepackage{hyperref}

\usepackage{enumerate}
\usepackage[shortlabels]{enumitem}

\usepackage{footmisc}

\usepackage{multirow}
\usepackage{array}
\usepackage{tabularx}
\newcolumntype{H}{>{\setbox0=\hbox\bgroup}c<{\egroup}@{}} 
\newcolumntype{L}[1]{>{\raggedright\let\newline\\\arraybackslash\hspace{0pt}}m{#1}}
\newcolumntype{C}[1]{>{\centering\let\newline\\\arraybackslash\hspace{0pt}}m{#1}}
\newcolumntype{R}[1]{>{\raggedleft\let\newline\\\arraybackslash\hspace{0pt}}m{#1}}
\newcolumntype{S}[1]{>{\raggedright\let\newline\\\arraybackslash\hspace{0pt}\tiny}m{#1}}

\usepackage{pifont}

\begin{document}

\title{\Large A Framework to Adjust Dependency Measure Estimates for Chance}
\author{
Simone Romano\thanks{CIS Department, The University of Melbourne, Australia
\texttt{\{simone.romano, vinh.nguyen, baileyj, karin.verspoor\} @unimelb.edu.au}}
\and
Nguyen Xuan Vinh$^*$
\and 
James Bailey$^*$
\and
Karin Verspoor$^*$
}
\date{}


\maketitle

\begin{abstract}

Estimating the strength of dependency between two variables is fundamental for exploratory analysis and many other applications in data mining. For example: non-linear dependencies between two continuous variables can be explored with the Maximal Information Coefficient (MIC); and categorical variables that are dependent to the target class are selected using Gini gain in random forests. Nonetheless, because dependency measures are estimated on finite samples, the interpretability of their quantification and the accuracy when ranking dependencies become challenging. Dependency estimates are not equal to 0 when variables are independent, cannot be compared if computed on different sample size, and they are inflated by chance on variables with more categories. In this paper, we propose a framework to adjust dependency measure estimates on finite samples. Our adjustments, which are simple and applicable to any dependency measure, are helpful in improving interpretability when quantifying dependency and in improving accuracy on the task of ranking dependencies. In particular, we demonstrate that our approach enhances the interpretability of MIC when used as a proxy for the amount of noise between variables, and to gain accuracy when ranking variables during the splitting procedure in random forests.

\end{abstract}

\section{Introduction}

Dependency measures $\mathcal{D}(X,Y)$ are employed in data mining to assess the strength of the dependency between two continuous or categorical variables $X$ and $Y$. If the variables are continuous, we can use Pearson's correlation to detect linear dependencies, or use more sophisticated measures, such as the Maximal Information Coefficient (MIC)~\cite{Reshef2011} to detect \emph{non}-linear dependencies. If the variables are categorical we can use the well known mutual information (a.k.a.\ information gain) or the Gini gain~\cite{Kononenko1995}. Dependency measures are ubiquitously used: to infer biological networks \cite{Reshef2011}, for variable selection for classification and regression tasks~\cite{Guyon2003}, for clustering comparisons and validation~\cite{Romano2015}, as splitting criteria in random forest~\cite{Breiman2001}, and to evaluate classification accuracy~\cite{Witten2011}, to list a few.

Nonetheless, there exist a number of problems when the dependency $\mathcal{D}(X,Y)$ is estimated with $\hat{\mathcal{D}}(\mathcal{S}_n|X,Y)$ on a data sample $\mathcal{S}_n$ of $n$ data points: \emph{a)} even if the population value $\mathcal{D}(X,Y) = 0$ when $X$ and $Y$ are statistically independent, estimates have a high chance to be bigger than 0 when $n$ is finite; \emph{b)} when comparing pairs of variables which share the same fixed population value $\mathcal{D}(X,Y)$, estimates are still dependent on the sample size $n$ and the number of categories of $X$ and $Y$.
These issues diminish the utility of dependency measures on \emph{quantification} tasks. For example, MIC was proposed in~\cite{Reshef2011} as a proxy of the amount of noise on the functional dependence between $X$ and $Y$: it should ``provide a score that roughly equals the coefficient of determination $R^2$ of the data relative to the regression function'', which is 0 under complete noise and 1 in noiseless scenarios. Nonetheless, MIC is not equal to 0 under complete noise, and MIC values are not comparable if computed on samples of different size $n$ because of the use of different datasets or in the case of variables with missing values:
\begin{example} \label{ex:mic}
Given two uniform and independent variables $X$ and $Y$ in $[0,1]$, the population value of \emph{MIC} is $0$ but the estimates $\mbox{\emph{MIC}}(\mathcal{S}_{20}|X,Y)$ on $20$ data points are higher than $\mbox{\emph{MIC}}(\mathcal{S}_{80}|X,Y)$ on $80$ data points. On average, they achieve the values of $0.36$ and $0.25$ respectively. The user expects this value to be $0$. The following box plots show estimates for $10,000$ simulations.
\includegraphics[scale=.45]{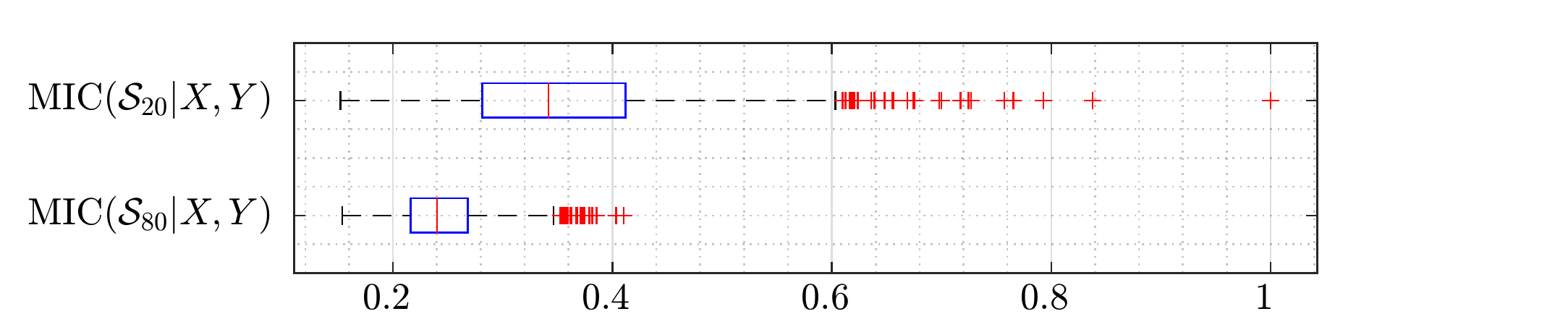}
\end{example}
The example above shows that the estimated MIC does not have zero baseline for finite samples. The zero baseline property is well known in the clustering community~\cite{Romano2015}, nonetheless this property does not hold for many dependency measures used in data mining.

Problems also arise when \emph{ranking} dependencies on a finite data sample. For example, if Gini gain is used to rank the dependency between variables to the target class in random forests~\cite{Breiman2001}, variables with more categories have more chances to be ranked higher:
\begin{example} \label{ex:gini}
Given a variable $X_1$ with two categories and a variable $X_2$ with one more category which are both independent of the target binary class $Y$, both the population value of Gini gain between $X_1$ and $Y$, and the population value between $X_2$ and $Y$ are equal to $0$. However, when Gini gain is estimated on $100$ data points the probability of $\mbox{\emph{Gini}}(\mathcal{S}_{100}|X_2,Y)$ being greater than $\mbox{\emph{Gini}}(\mathcal{S}_{100}|X_1,Y)$ is equal to $0.7$. The user expects $0.5$ given that $X_1$ and $X_2$ are equally unpredictive to $Y$.
\end{example}
It is common practice to use the $p$-value of Gini gain to correct this bias~\cite{Dobra01}. Nonetheless, we will shortly see that $p$-values are effective only when the population value of a dependency measure is $0$. 

In this paper, we identify that the issues discussed in Example~\ref{ex:mic} and~\ref{ex:gini} are due to inflated estimates arising from finite samples. 
Statistical properties of the distribution of the dependency measure estimator $\hat{\mathcal{D}}(\mathcal{S}_n|X,Y)$ under independence of $X$ and $Y$ can be used to adjust these estimates. The challenge is to formalize a general framework to adjust dependency measure estimates which also addresses the shortcomings of the use of $p$-values. We make the following contributions:
\begin{itemize}[topsep=1ex,itemsep=-1ex,partopsep=1ex,parsep=1ex] 
\item We identify common biases of dependency measure estimates due to finite samples;
\item We propose a framework to adjust estimates $\hat{\mathcal{D}}(\mathcal{S}_n|X,Y)$ which is simple,  yet applicable to many dependency measures because it only requires to use the distribution of the estimator when $X$ and $Y$ are independent;
\item We experimentally demonstrate that our adjustments improve interpretability when quantifying dependency (e.g.,\ when using MIC as a proxy of the amount of noise) and accuracy when ranking dependencies (e.g.,\ when using Gini gain in random forests).
\end{itemize}

\section{Background}

Dependency measures $\mathcal{D}(X,Y)$ are defined on the joint distribution $(X,Y)$. In data mining applications, they are estimated with $\hat{\mathcal{D}}(\mathcal{S}_n|X,Y)$ on a finite sample $\mathcal{S}_n = \{(x_k,y_k)\}$ of $n$ data points. If variables are continuous, we can compute the amount of linear dependency with the squared Pearson's correlation coefficient:
\begin{equation}
r^2(\mathcal{S}_n|X,Y) \triangleq \frac{ \Big( \sum_{k=1}^n (x_k - \bar{x})(y_k - \bar{y}) \Big)^2 }{\sum_{k=1}^n (x_k - \bar{x})^2 \sum_{k=1}^n (y_k - \bar{y})^2}
\end{equation}
with $ \bar{x} = \frac{1}{n}\sum_{k=1}^n x_k$ and $\bar{y} = \frac{1}{n}\sum_{k=1}^n y_k$. If we are interested in \emph{non}-linear relationships, we can employ the Maximal Information Coefficient (MIC)~\cite{Reshef2011}. $\mbox{MIC}(\mathcal{S}_n|X,Y)$ is estimated as the maximum normalized mutual information across all the possible grids superimposed on the sample $\mathcal{S}_n$ to estimate the joint distribution of $X$ and $Y$. When the variables are categorical, the mutual information or $\mbox{Gini}(\mathcal{S}_n|X,Y)$ can be directly estimated using the joint empirical probability distribution between $X$ and $Y$ on the sample $\mathcal{S}_n$.
See Appendix~\ref{app:def} in the supplement for formal definitions.

There are three important applications of dependency measures between two variables~\cite{Reimherr2013}:
\begin{description}[topsep=1ex,itemsep=-1ex,partopsep=1ex,parsep=1ex] 
\item[\textit{\textbf{Detection}}:] Test for the presence of dependency. For example, assess if there exists any dependence between bacterial species that colonize the gut of mammals~\cite{Reshef2011};
\item[\textit{\textbf{Quantification}}:] Summarization of the amount of dependency in an interpretable fashion. For example, when MIC is used as a proxy of the amount of noise in a relationship~\cite{Reshef2011};
\item[\textit{\textbf{Ranking}}:] Sort the relationships of different variables based on the strength of their dependency. For example, when Gini gain is used to rank predictive variables to the target class in random forests~\cite{Breiman2001}.
\end{description}
We saw in Examples~\ref{ex:mic} and~\ref{ex:gini} that when it comes to estimating dependency on data samples via $\hat{\mathcal{D}}(\mathcal{S}_n|X,Y)$ the interpretability of \emph{quantification} and accuracy of \emph{ranking} become challenging. We claim that both tasks can take advantage of the distribution of $\hat{\mathcal{D}}$ under the following null hypothesis:
\begin{definition}
$\hat{\mathcal{D}}_0(\mathcal{S}_n|X,Y)$ is the distribution of $\hat{\mathcal{D}}(\mathcal{S}_n|X,Y)$ on a sample $\mathcal{S}_n$ under the \textbf{null hypothesis} that $X$ is statistically independent of $Y$.
\end{definition}
This null hypothesis is commonly exploited only in \emph{detection} tasks where the distribution $\hat{\mathcal{D}}_0(\mathcal{S}_n|X,Y)$ is computed under the null 
and a $p$-value is computed to filter out false discoveries~\cite{Reshef2011}. Nonetheless, this null can be used also to aid \emph{quantification} and \emph{ranking}. The challenges are to identify the distribution under the null for a particular dependency measure, and to employ it in a framework to perform adjustments to the estimates. Here we discuss the use of this null hypothesis in previous research.

\subsection{Use of the Null for Quantification.} \label{sec:adjquant}

To our knowledge the first instance of a systematic approach using the null distribution $\hat{\mathcal{D}}_0$ to achieve interpretability in quantification was proposed in the 1960 with the $\kappa$ coefficient of inter-annotator agreement~\cite{Witten2011}. The amount of agreement $A(\mathcal{S}_n|X,Y)$ (dependency) between two annotators $X$ and $Y$ on a sample of $n$ items can be adjusted for chance by subtracting its expected value $E[A_0(\mathcal{S}_n|X,Y)]$ under the null hypothesis of independence between annotators. The $\kappa$ coefficient is obtained by normalization via division of its maximum value $\max{A} = 1$ to obtain an adjusted dependency measure in the range $[0,1]$:
\begin{equation} \label{eq:kappa}
\kappa(\mathcal{S}_n|X,Y) = \frac{ A(\mathcal{S}_n|X,Y) - E[A_0(\mathcal{S}_n|X,Y)] }{1 - E[A_0(\mathcal{S}_n|X,Y)] }
\end{equation}
Other notable examples are the Adjusted Rand Index (ARI)
and the Adjusted Mutual Information (AMI)~\cite{Romano2015}. We argue that this approach should be applied to many other dependency measures estimators $\hat{\mathcal{D}}$ because it improves interpretability by guaranteeing a zero baseline to $\hat{\mathcal{D}}$. Moreover, we will shortly see that it helps in comparing estimates on different samples $\mathcal{S}_n$.

\subsection{Use of the Null for Ranking.}

In the decision tree community, it is very well known that when selecting the most dependent variable $X$ to the target class $Y$, variables available on a small number of samples $n$ or with many categories tend to be chosen more often. Indeed, it has been shown that an unbiased selection can be obtained if the $p$-value~\cite{Dobra01} of a dependency estimate or its standardized version~\cite{Romano2014} is used rather than its raw value. Nonetheless, these techniques are unbiased only under the null hypothesis and not unbiased in general. Indeed, in the next sections we will see that their use actually yields bias towards variables induced on bigger $n$ or with fewer categories. This behavior has been overlooked in the decision tree community.

\section{Adjusting Estimates for Quantification} \label{sec:willadjquant}


To guarantee good interpretability in quantification tasks, dependency measure estimates should be equal to 0 on average when $X$ and $Y$ are independent, and their values should be comparable on average across different data samples of different size. More formally we want:
\begin{proper}[Zero Baseline]\label{prop:zerobas}
If $X$ and $Y$ are independent then $E[\hat{\mathcal{D}}(\mathcal{S}_n|X,Y)]=0$ for all $n$.
\end{proper}
\begin{proper}[Quantification Unbiasedness]\label{prop:qunbias}
If $\mathcal{D}(X_1,Y_1) = \mathcal{D}(X_2,Y_2)$ then $E[\hat{\mathcal{D}}(\mathcal{S}_n|X_1,Y_1)] = E[\hat{\mathcal{D}}(\mathcal{S}_m|X_2,Y_2)]$ for all $n$ and $m$.	
\end{proper}
We saw in Example~\ref{ex:mic} that MIC does not satisfy either property.
Therefore, we propose an adjustment that can be applied to MIC and in general to any dependency estimator $\hat{\mathcal{D}}$:
\begin{definition}[Adjustment for Quantification] \label{def:adjquant}
\[
\mbox{\emph{A}}\hat{\mathcal{D}}(\mathcal{S}_n|X,Y) \triangleq \frac{\hat{\mathcal{D}}(\mathcal{S}_n|X,Y) - E[\hat{\mathcal{D}}_0(\mathcal{S}_n|X,Y)]}{\max{\hat{\mathcal{D}}(\mathcal{S}_n|X,Y)} - E[\hat{\mathcal{D}}_0(\mathcal{S}_n|X,Y)]}
\]
is the adjustment of $\hat{\mathcal{D}}(\mathcal{S}_n|X,Y)$, where $\max{\hat{\mathcal{D}}(\mathcal{S}_n|X,Y)}$ and $E[\hat{\mathcal{D}}_0(\mathcal{S}_n|X,Y)]$ are respectively the maximum of $\hat{\mathcal{D}}$, and its expected value under the null.
\end{definition} 
$\mbox{A}\hat{\mathcal{D}}(\mathcal{S}_n|X,Y)$ has always zero baseline (Property~\ref{prop:zerobas}) being 0 on average when $X$ and $Y$ are independent, and attains 1 as maximum value.
This adjustment can be applied to $r^2$ and MIC to increase their interpretability when they are used as proxies of the amount of noise in a linear relationship and a functional relationship respectively. 
We just have to identify their distribution on the sample $\mathcal{S}_n$ under the null:
\begin{itemize}[topsep=1ex,itemsep=-1ex,partopsep=1ex,parsep=1ex] 
\item $r^2_0(\mathcal{S}_n|X,Y)$: follows a Beta distribution with parameters $\frac{1}{2}$ and $\frac{n-2}{2}$\cite{Giles};
\item $\mbox{MIC}_0(\mathcal{S}_n|X,Y)$: this distribution can be computed using $s=1,\dots,S$ Monte Carlo permutations $\mbox{MIC}^{(s)}_0$ of MIC~\cite{Reshef2011}. See Appendix~\ref{app:nullmic}.
\end{itemize}

\noindent Therefore the adjusted Pearson's correlation squared $r^2$ and the adjusted MIC are:
\begin{equation} \label{eq:ar2}
\mbox{A}r^2(\mathcal{S}_n|X,Y) = \frac{r^2(\mathcal{S}_n|X,Y) - \frac{1}{n-1}}{1 - \frac{1}{n-1}}
\end{equation}
\begin{equation} \label{eq:amic}
\mbox{AMIC}(\mathcal{S}_n|X,Y) = \frac{\mbox{MIC}(\mathcal{S}_n|X,Y) - \mbox{EMIC}_0 }{1 - \mbox{EMIC}_0} 
\end{equation}
where $E[r^2_0(\mathcal{S}_n|X,Y)] = \frac{1}{n-1}$ and $\mbox{EMIC}_0 = \frac{1}{S} \sum_{s = 1}^S \mbox{MIC}_0^{(s)}$.
$\mbox{EMIC}_0$ converges to $E[\mbox{MIC}_0(\mathcal{S}_n|X,Y)]$ at the limit of infinite permutations. However, good estimation accuracy can be obtained even with few permutations because of the law of large numbers~\cite{Good2005}.


In the next section we will see how our adjustments satisfy Property~\ref{prop:zerobas} and Property~\ref{prop:qunbias}.

\subsection{Experiments with Pearson Correlation and MIC.}

We aim to experimentally verify the zero baseline Property~\ref{prop:zerobas} and that our adjustment in Definition~\ref{def:adjquant} enables better \emph{interpretability}. We generate a linear relationship between a uniformly distributed $X$ in $[0,1]$ and $Y$ on $n=30$ points adding different percentages of white noise. We compare $r^2$ and $\mbox{A}r^2$. Each white noise level is obtained by substituting a given percentage of points from the relationship and assigning to the $Y$ coordinate a random value in $[0,1]$. Figure~\ref{fig:noise_r2} shows the average $r^2$  and $\mbox{A}r^2$ for 2,000 simulated relationships with a given percentage of white noise: $r^2$ is not zero on average when the amount of noise is 100\% (last plot on the right). On the other hand, $\mbox{A}r^2$ is very close to zero when there is complete noise and it fully exploits its range of values from one to zero, mapping the domain from 0\% to 100\% noise. This yields more interpretability and enables $\mbox{A}r^2$ to be used as a \emph{proxy to quantify the amount of noise in linear relationships}.

Similarly, we generated a quadratic relationship between $X$ and $Y$ in $[0,1]\times [0,1]$ on $n=60$ points with different levels of noise to compare MIC and AMIC. Figure~\ref{fig:noise_MIC} shows that the value of MIC computed with default parameters~\cite{Reshef2011}, is about 0.26 on average for complete noise. AMIC computed with $S = 30$ permutations, is instead very close to zero and it exploits better its range of values from one to zero. AMIC is more interpretable than MIC and might be used more intuitively as a \emph{proxy for the amount of noise in a functional relationship}.
\begin{figure*}[t]
\hspace*{-.7cm}  
\includegraphics[scale=.65]{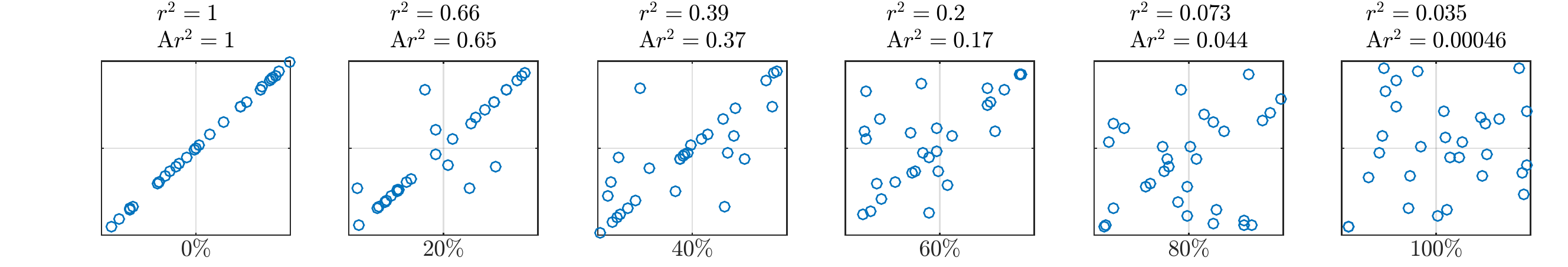}
\caption{Average value of $r^2$ and A$r^2$ for different percentages of white noise. \emph{Linear} relationship between $X$ and $Y$ induced on $n=30$ points in $[0,1]\times[0,1]$. A$r^2$ becomes zero on average on 100\% noise enabling a more interpretable range of variation.} \label{fig:noise_r2}
\vspace{-.5em} 
\end{figure*}
\begin{figure*}[t]
\hspace*{-.7cm}  
\includegraphics[scale=.65]{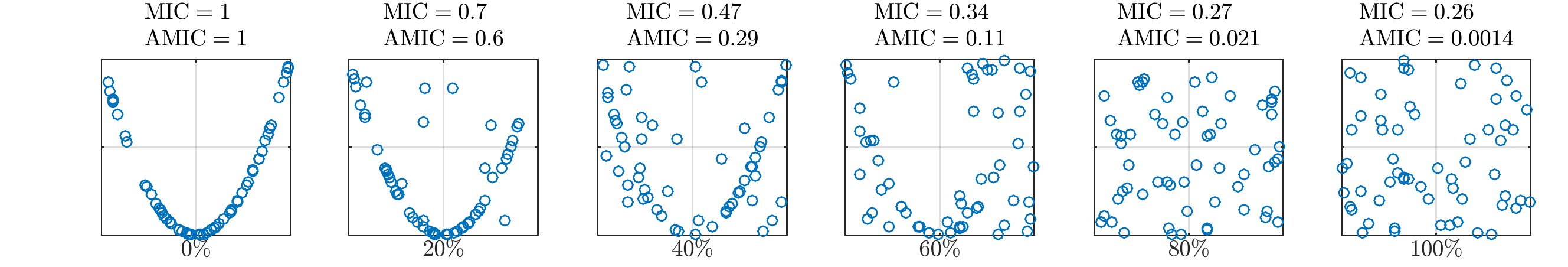}
\caption{Average value of MIC and AMIC for different percentages of white noise. \emph{Quadratic} relationship between $X$ and $Y$ induced on $n=60$ points in $[0,1]\times[0,1]$. AMIC becomes zero on average on 100\% noise enabling a more interpretable range of variation.} \label{fig:noise_MIC}
\vspace{-.7em} 
\end{figure*}

The average value of a dependency estimator should not be biased with regards to the sample $\mathcal{S}_n$ as stated in Property~\ref{prop:qunbias}. In Figure~\ref{fig:noiseprc}, we show that $r^2$ and MIC suffer from this problem: their estimates are higher on average when $n$ is smaller. Figure~\ref{fig:noiseprc} shows the average value of raw and adjusted measures on 2,000 simulations for different levels of noise and sample size $n$: $r^2$ and $\mbox{A}r^2$ are compared on linear relationships; MIC and AMIC are compared on linear, quadratic, cubic, and 4th root relationships.
Neither the zero baseline Property~\ref{prop:zerobas} nor the quantification unbiasedness Property~\ref{prop:qunbias} is verified for the raw measures $r^2$ and MIC, shown respectively in Figure~\ref{fig:eqr2} and~\ref{fig:eqmic}.
Instead, $\mbox{A}r^2$ and AMIC in Figure~\ref{fig:eqar2} and~\ref{fig:eqamic}, satisfy both properties: they have zero baseline and their average value is not biased with regards to the sample size $n$. We claim that these properties improve \emph{interpretability} when quantifying dependency and \emph{enhance equitability} for MIC~\cite{Reshef2011}.
\begin{figure*}
\centering     
\subfigure[$r^2$ (raw)]{\label{fig:eqr2}\includegraphics[scale=.5]{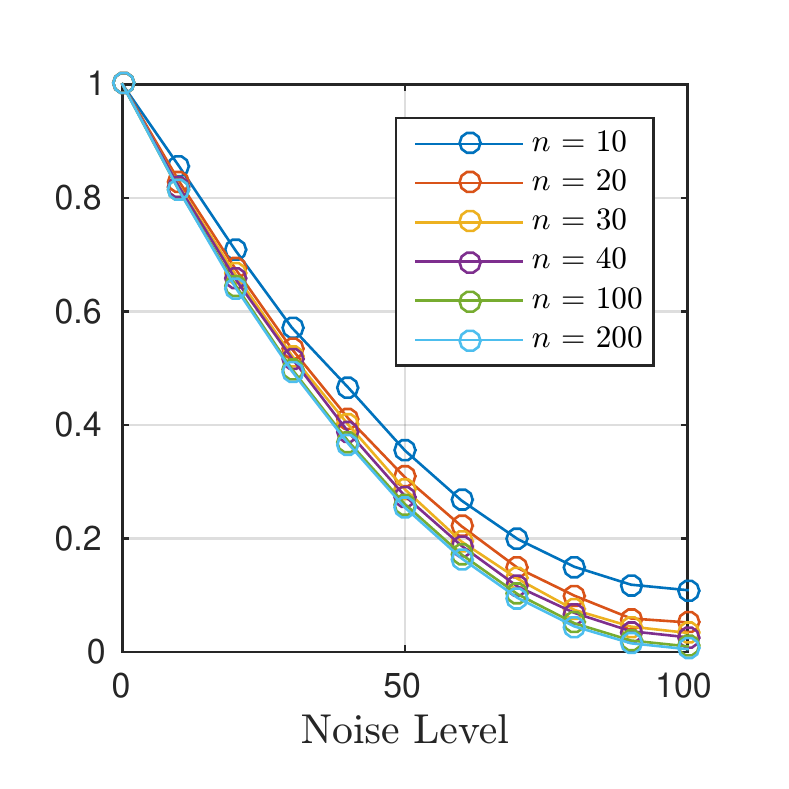}}
\subfigure[$\mbox{A}r^2$ (\emph{adjusted})]{\label{fig:eqar2}\includegraphics[scale=.5]{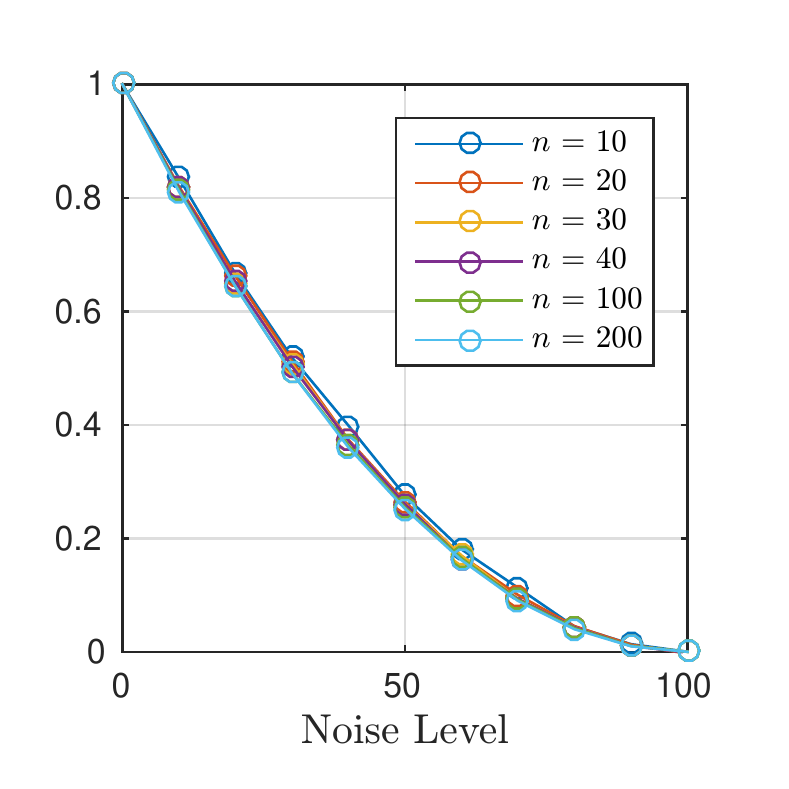}}
\subfigure[MIC (raw)]{\label{fig:eqmic}\includegraphics[scale=.5]{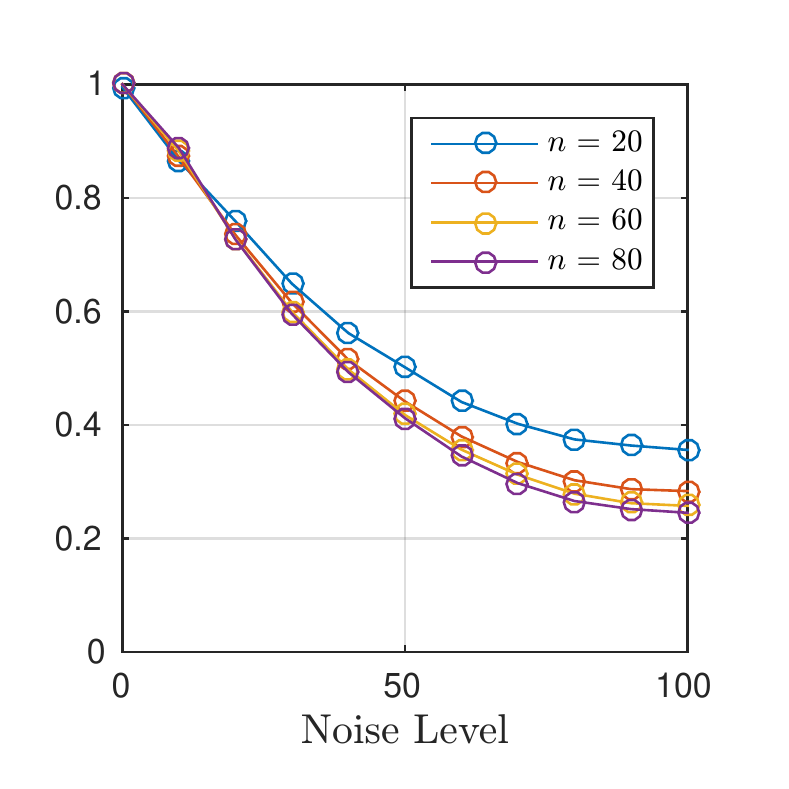}}
\subfigure[AMIC (\emph{adjusted})]{\label{fig:eqamic}\includegraphics[scale=.5]{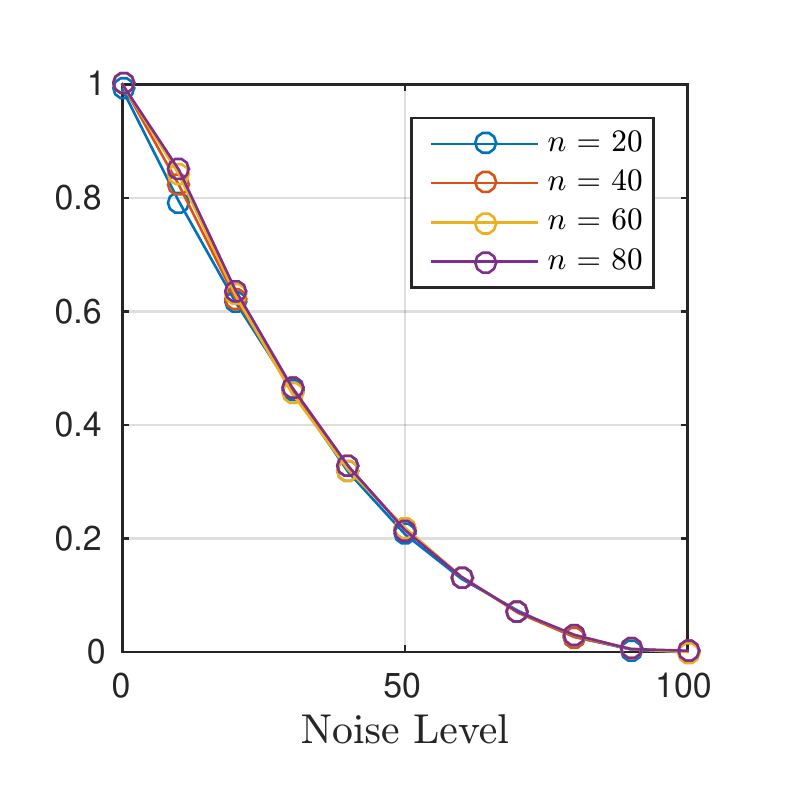}}
\caption{Average value of $r^2$, $\mbox{A}r^2$, MIC, and AMIC on different amount of noise and different sample size $n$.
Raw measures show higher values for smaller $n$ on average. Instead, Property~\ref{prop:qunbias} of unbiasedness with regards to $n$ is empirically verified for \emph{adjusted} measures.}\label{fig:noiseprc}
\vspace{-1em}
\end{figure*}

\section{Adjusting Estimates for Ranking}

When the task is ranking dependencies according to their strength, dependencies induced on smaller sample size $n$ or on variables with more categories have more chances to be ranked higher as shown in Example~\ref{ex:gini} for Gini gain. This issue is due to inflated estimates due to finite samples. Indeed, $r^2$ and MIC suffer from the same problem.

Consider this experiment: we generate five samples $\mathcal{S}_n$ with $n = [20,40,60,80,100]$ to simulate different amount of missing values for a joint distribution $(X,Y)$ where $X$ and $Y$ are independent. For each sample, we compute $r^2(\mathcal{S}_n|X,Y)$, we select $\mathcal{S}_n$ that achieves the highest value, and iterate this process 10,000 times. Given that the population value $\rho^2(X,Y) = 0$ for all samples, all samples should have equal chances to maximize the $r^2$. However, Figure~\ref{fig:rhoWinEx} shows that $\mathcal{S}_{20}$ has higher chances to maximize $r^2$. \emph{This implies that dependencies estimated on samples with missing values have higher chances to be ranked higher in terms of strength.}
\begin{figure}[h]
\centering
\includegraphics[scale=.7]{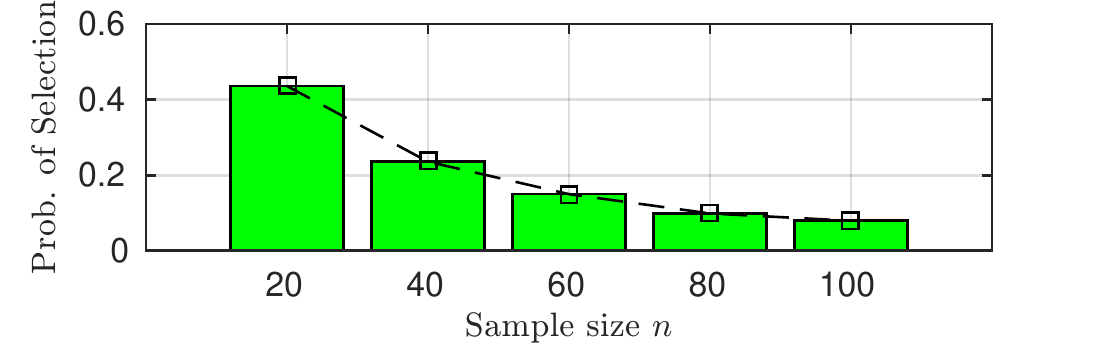}
\caption{Probability to select the sample $\mathcal{S}_n$ with $n = [20,40,60,80,100]$ according $r^2(\mathcal{S}_n|X,Y)$ fixing the population value $\rho^2(X,Y) = 0$. The relationship with $n=20$ has more chances to be ranked higher.} \label{fig:rhoWinEx}
\vspace{-.5em} 
\end{figure}

\noindent We would like that dependencies which share the same population value for $\mathcal{D}$ had the same chances to maximize the dependency estimate $\hat{\mathcal{D}}$ even if estimated on different samples. More formally:
\begin{proper}[Ranking Unbiasedness] \label{prop:rankingunbias}
If $\mathcal{D}(X_1,Y_1) = \mathcal{D}(X_2,Y_2) = ..  = \mathcal{D}(X_K,Y_K)$ then the probability of $\hat{\mathcal{D}}(\mathcal{S}_{n_i}|X_i,Y_i)$ being equal or greater than any $\hat{\mathcal{D}}(\mathcal{S}_{n_j}|X_j,Y_j)$ is $\frac{1}{K}$ for all $n_i,n_j$, $1\leq i \neq j \leq K$.
\end{proper}
For example in Figure~\ref{fig:rhoWinEx} we would like constant probability of selection equal to $\frac{1}{5}= 0.20$.
Property~\ref{prop:rankingunbias} is useful to achieve \emph{higher accuracy} when the task is ranking the pair of variables that show the stronger relationship. 


Biases in ranking are well known in the decision tree community~\cite{Dobra01} as shown in Example~\ref{ex:gini}. Distributional properties of the raw dependency measure have to be employed to adjust for biases in ranking. For example, ranking according to $p$-values or standardized measures are possible solutions~\cite{Dobra01,Romano2014}. They both quantify if the estimate $\hat{\mathcal{D}}$ is statistically significant. Here we extend the standardization technique to any dependency measure estimate $\hat{\mathcal{D}}$ to employ it for unbiased ranking:
\begin{definition}[Standardization for Ranking] \label{def:std}
\[
\mbox{\emph{S}}\hat{\mathcal{D}}(\mathcal{S}_n|X,Y) \triangleq \frac{\hat{\mathcal{D}}(\mathcal{S}_n|X,Y) - E[\hat{\mathcal{D}}_0(\mathcal{S}_n|X,Y)]}{\sqrt{\mbox{\emph{Var}}(\hat{\mathcal{D}}_0(\mathcal{S}_n|X,Y))}}
\]
is the standardized $\hat{\mathcal{D}}(\mathcal{S}_n|X,Y)$, where $E[\hat{\mathcal{D}}_0(\mathcal{S}_n|X,Y)]$ and $\mbox{\emph{Var}}(\hat{\mathcal{D}}_0(\mathcal{S}_n|X,Y))$ are, respectively, the expected value and the variance of $\hat{\mathcal{D}}$ under the null.
\end{definition}
Nonetheless, it is very difficult to satisfy the ranking unbiasedness Property~\ref{prop:rankingunbias} just with $\mbox{S}\hat{\mathcal{D}}$. Therefore we also define an adjustment to dependency measures whose bias can be tuned according to a parameter $\alpha$. This is particularly useful when $\alpha$ can be tuned with cross-validation, e.g.\ in random forests.
\begin{definition}[Adjustment for Ranking]
\[
\mbox{\emph{A}}\hat{\mathcal{D}}(\mathcal{S}_n|X,Y)(\alpha) \triangleq \hat{\mathcal{D}}(\mathcal{S}_n|X,Y) - q_0(1-\alpha)
\]
is the adjustment at level $\alpha \in (0,1]$ of $\hat{\mathcal{D}}(\mathcal{S}_n|X,Y)$, where $q_0(1-\alpha)$ is the $(1-\alpha)$-quantile of $\hat{\mathcal{D}}_0(\mathcal{S}_n|X,Y)$ under the null: i.e., $P\Big(\hat{\mathcal{D}}(\mathcal{S}_n|X,Y) \leq q_0(1-\alpha) \Big) = 1 - \alpha$.
\end{definition}
At a fixed significance level $\alpha$, the quantile $q_0(1-\alpha)$ induces more penalization when the estimate is not statistically significant. With regards to Example~\ref{ex:gini}, fixing $\alpha = 0.05$ we penalize the variable $X_1$ and the variable $X_2$ by $q_0(0.95)$ equal to 0.036 and 0.053 respectively. The latter variable gets penalized more because it is less statistically significant having more categories. In contrast, $\mbox{S}\hat{\mathcal{D}}$ fixes the amount of penalization based on statistical significance and does not allow to tune the bias during ranking. In the next section we aim to show the shortcomings of raw measures and standardized measures for ranking tasks.

\begin{figure*}
\centering     
\subfigure[$X$ independent of $Y$ ($\rho^2 = 0$).]{\label{fig:100}\includegraphics[scale=.7]{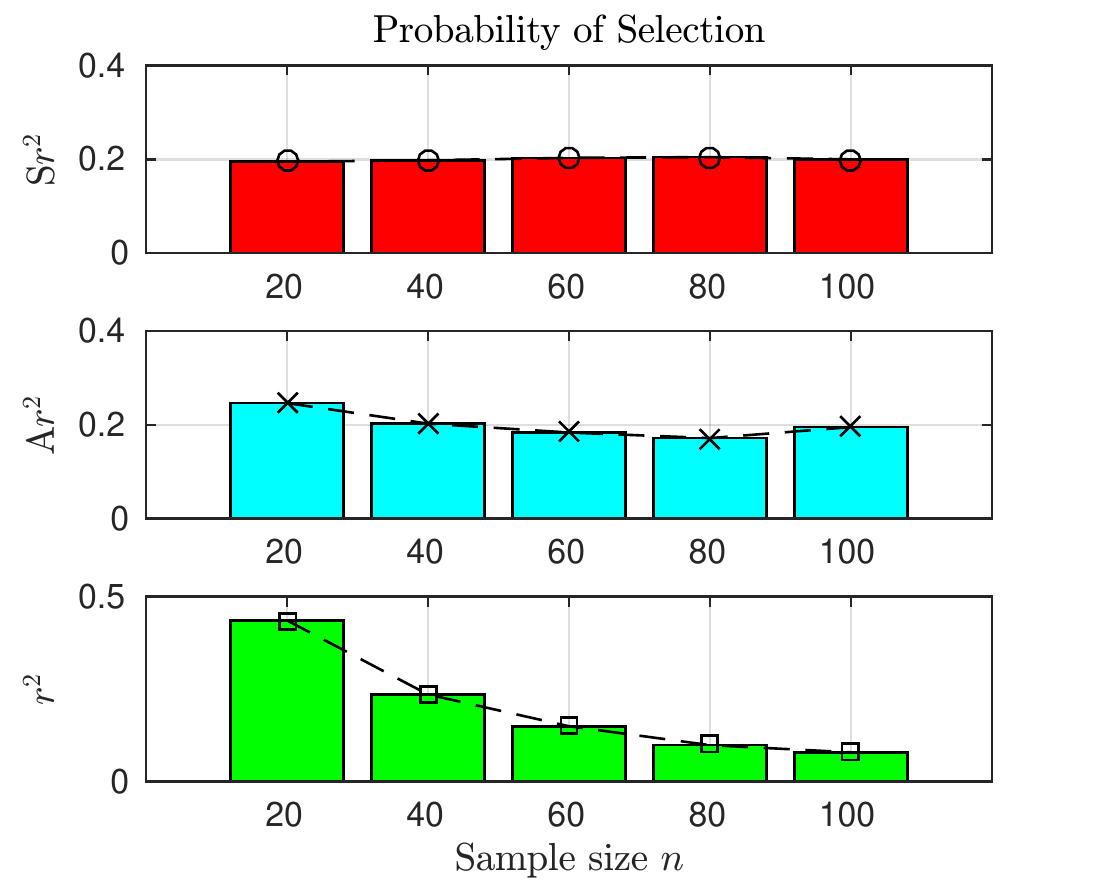}}
\subfigure[$X$ linearly related to $Y$ with 10\% white noise ($\rho^2 > 0$).]{\label{fig:10}\includegraphics[scale=.7]{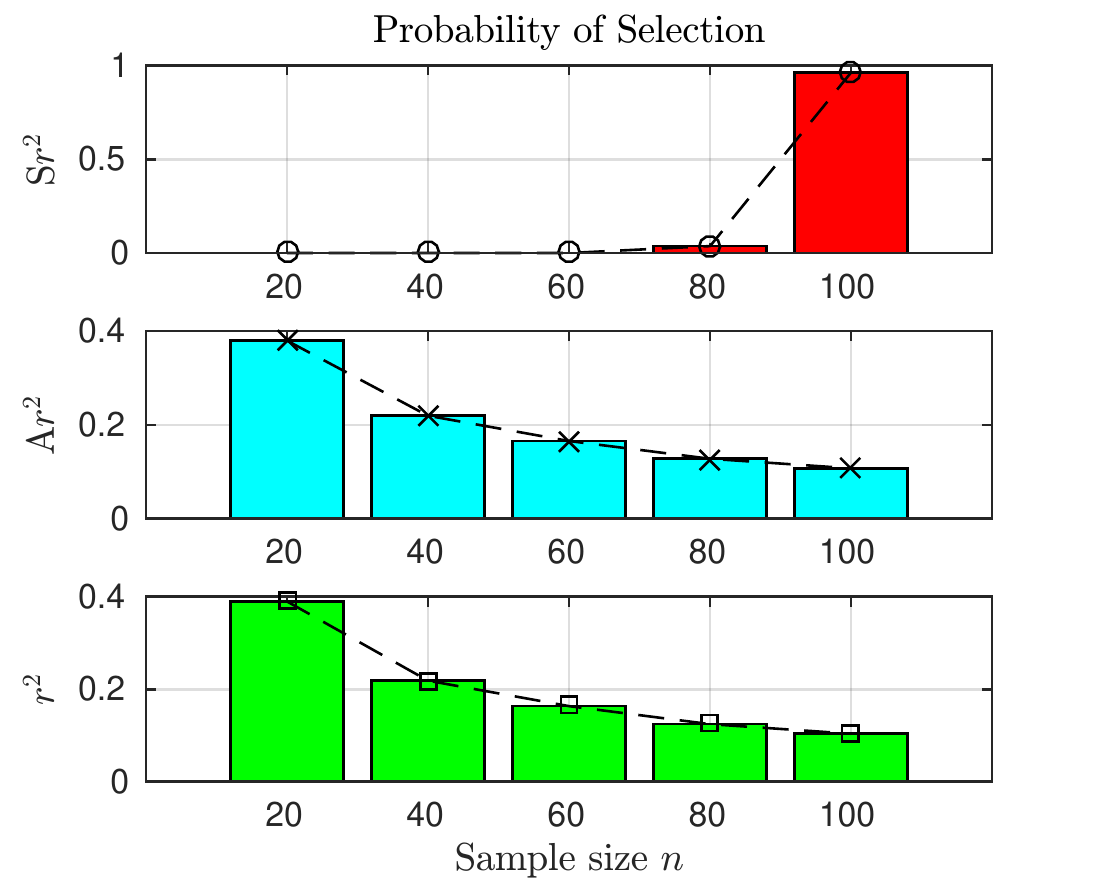}}
\caption{Probability to select the sample $\mathcal{S}_n$ induced on $n = [20,40,60,80,100]$ according adjusted measures: $\mbox{S}r^2$ satisfies the ranking unbiasedness Property~\ref{prop:rankingunbias} when $\rho^2 = 0$ but not when $\rho^2 > 0$. All measures show to be biased in the latter case: it is difficult to satisfy Property~\ref{prop:rankingunbias} in general.}\label{fig:selbias}
\vspace{-1em}
\end{figure*}

\subsection{Ranking Biases of Raw and Standardized Measures.}

We use $r^2$ and its adjusted versions in a case study: $\mbox{A}r^2$ is defined as per Eq.~\eqref{eq:ar2}, $\mbox{A}r^2(\alpha) = r^2 - q_0(1-\alpha)$ where $q_0(1-\alpha)$ is computed with the Beta distribution (see Section~\ref{sec:willadjquant}), and the standardized $r^2$ is defined as: 
\begin{equation}
\mbox{S}r^2(\mathcal{S}_n|X,Y) = \frac{r^2(\mathcal{S}_n|X,Y) - \frac{1}{n-1}}{\sqrt{\frac{2(n-2)}{(n-1)^2(n+1)}}}
\end{equation}
We do not evaluate $p$-values because their use is equivalent to the use of standardized measures which are also much easier to compute.

We perform similar experiments as in the previous section: we fix the population value for a dependency and compute estimates on different samples $\mathcal{S}_n$ to compute their probability of selection. We select samples according $r^2$, $\mbox{A}r^2$, and $\mbox{S}r^2$. Figure~\ref{fig:100} shows the probability of selection of different samples at fixed population value $\rho^2 = 0$. We can clearly see that the ranking unbiasedness Property~\ref{prop:rankingunbias} is satisfied if we use $\mbox{S}r^2$ (top plot). On the other hand the sole adjustment for quantification $\mbox{A}r^2$ is not enough to remove $r^2$ bias towards small $n$. Nonetheless, Figure~\ref{fig:10} shows that if we generate a linear relationship between $X$ and $Y$ with 10\% white noise (i.e.,\ $\rho^2$ is fixed to a value greater than $0$), $\mbox{S}r^2$ is biased towards big $n$. This is because we prefer statistically significant relationships. This phenomena might have been overlooked in the decision tree community~\cite{Strobl2007Gini,Frank98}.

Given that it is difficult to satisfy the ranking unbiasedness Property~\ref{prop:rankingunbias} in general, we show how $\alpha$ in our adjustment $\mbox{A}r^2(\alpha)$ might be used to tune the bias when it is possible. Figure~\ref{fig:biasalpha} shows that with big $\alpha$ ($\alpha \approx 0.4$) relationships on small $n$ have higher probability to be selected. On the other hand, small $\alpha$ ($\alpha \approx 0.05$) tunes the bias towards higher sample size $n$.  
\begin{figure}[h]
\centering
\includegraphics[scale=.7]{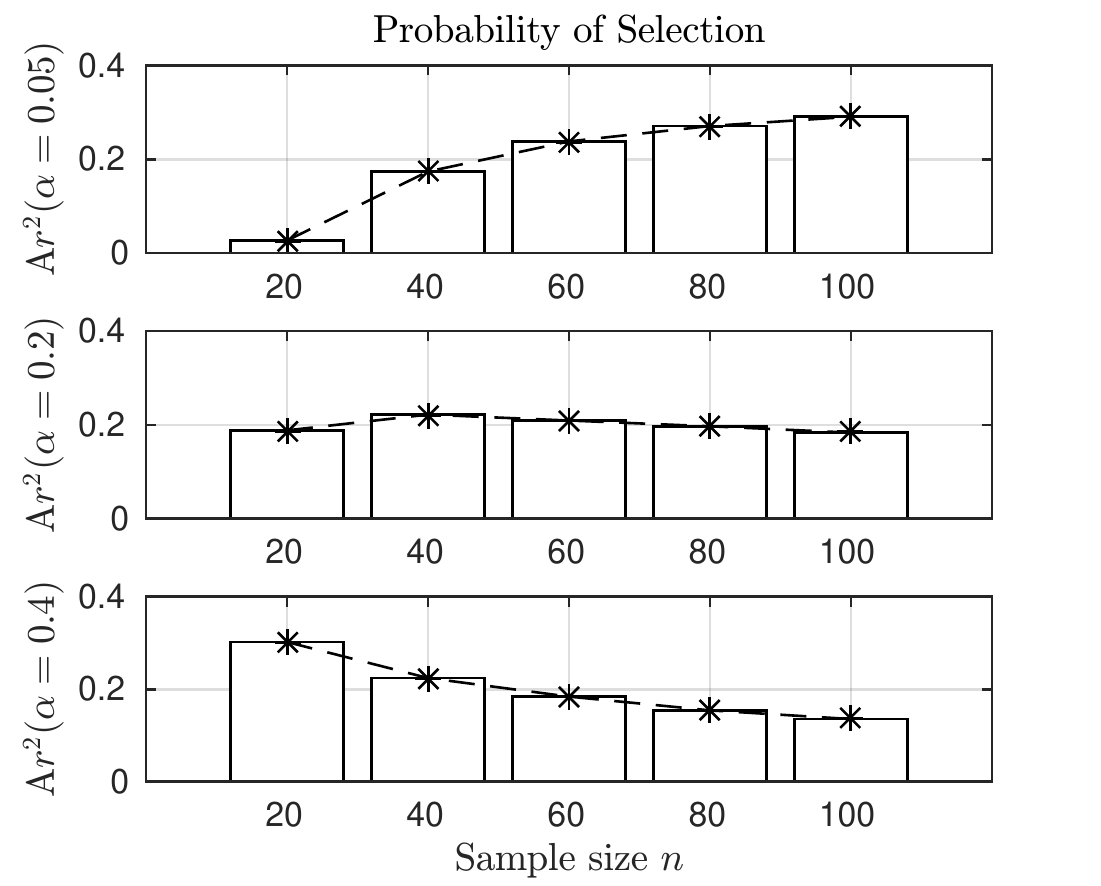}
\caption{Probability of selection of a sample $\mathcal{S}_n$ when $X$ is linearly related to $Y$ with 10\% white noise using $\mbox{A}r^2(\alpha)$: $\alpha$ tunes the bias towards small $n$ with a big $\alpha$ (bottom plot) or big $n$ with a small $\alpha$ (top plot).} \label{fig:biasalpha}
\vspace{-1.5em} 
\end{figure}
On a real ranking task, it is reasonable to rank according to $\mbox{A}r^2(\alpha)$ and see how the rank changes with changes of $\alpha$ rather than relying on a single ranking based on biased measures such as $r^2$, $\mbox{A}r^2$, or $\mbox{S}r^2$. The best value for $\alpha$ can be chosen by cross-validation when it is possible. Similar conclusions can be drawn for MIC and its adjusted versions.

\subsection{Experiments with Pearson Correlation and MIC.}

MIC and $r^2$ have been used in~\cite{Reshef2011} to identify the strongest related pair of socio-economic variables using the WHO dataset.
This dataset is a collection of $m = 357$ variables for $n = 201$ countries. Some of the variables have a high percentage of missing values and they are available on much fewer than $n = 201$ samples. In this section, we aim to alert the users of MIC and $r^2$ about ranking biases for relationships induced on different sample size $n$. We conduce an experiment: we choose a reference socio-economic variable $Y$ and select the top related variable according to $r^2$ and its adjusted versions. Then, we estimate the dependency between two variables based on the data points available for both $X$ and $Y$. We only consider dependencies estimated on at least $ n \geq 10$ data points. Figure~\ref{fig:exrho} shows the top-most dependent variable $X$ to $Y =$``\emph{Breast cancer number of female deaths}'' using $r^2$, $\mbox{A}r^2$, $\mbox{S}r^2$, and $\mbox{A}r^2(\alpha = 0.1)$. 
\begin{figure*}
\centering
\includegraphics[scale=.48]{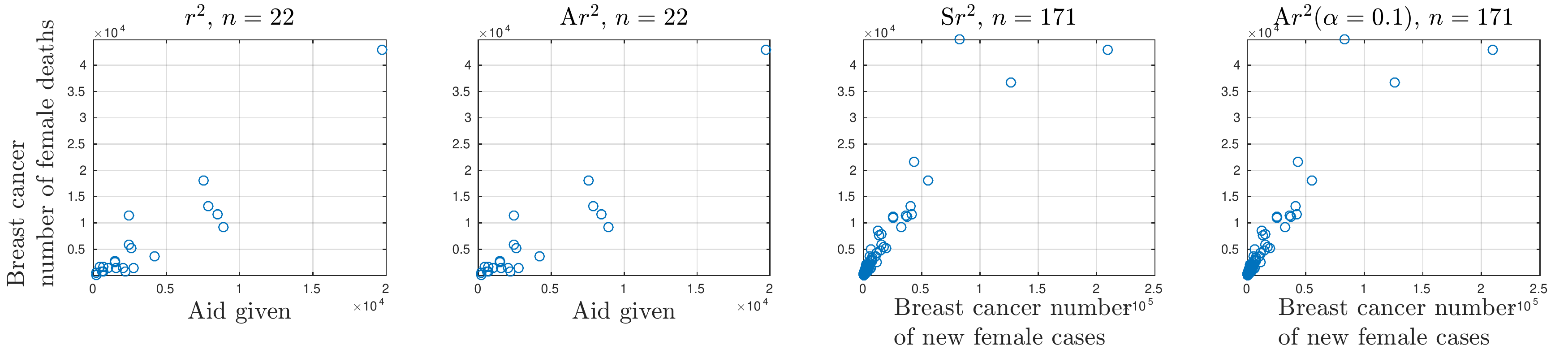}
\caption{Plot of the top-most dependent variable $X$ to $Y =$``\emph{Breast cancer number of female deaths}'' according different adjustments for $r^2$.
$r^2$ and $\mbox{A}r^2$ favor relationships on small $n$. S$r^2$, and A$r^2(\alpha =0.1)$ penalize relationships on small $n$ and select more reasonably $X =$``\emph{Breast cancer number of new female cases}''.}\label{fig:exrho}
\vspace{-1em}
\end{figure*}
The top-most dependent variable according to $r^2$ and $\mbox{A}r^2$ is $X =$``\emph{Aid given}'' which quantifies the amount of aid given to poor countries in million US\$. Instead, $\mbox{S}r^2$ and $\mbox{A}r^2(\alpha = 0.1)$ identify $X =$``\emph{Breast cancer number of new female cases}'' which seems a more reasonable choice given that the number of deaths might be correlated with new cancer cases. Indeed as seen in the previous Section, $r^2$ and $\mbox{A}r^2$ favour variables induced on small $n$. Moreover from the plot in Figure~\ref{fig:exrho} we see that they are very sensitive to extreme values or outliers: i.e.\ the United States show a very high number of deaths due to breast cancer $\approx 43{,}000$ in a year and a very high amount of aid given $\approx 20$ Billion US\$; this increases the chances for a high $r^2$ or $\mbox{A}r^2$.

MIC is even more inclined to select variables induced on small $n$. For example we see in Figure~\ref{fig:exmic} that if we target $Y =$``\emph{Maternal mortality}'' which quantifies the number of female deaths during pregnancy (out of 100,000 live births), and we choose MIC or AMIC to identify the top dependent variable, we get $X =$``\emph{Oil consumption per person}'' (tonnes per year).
\begin{figure*}
\centering
\includegraphics[scale=.48]{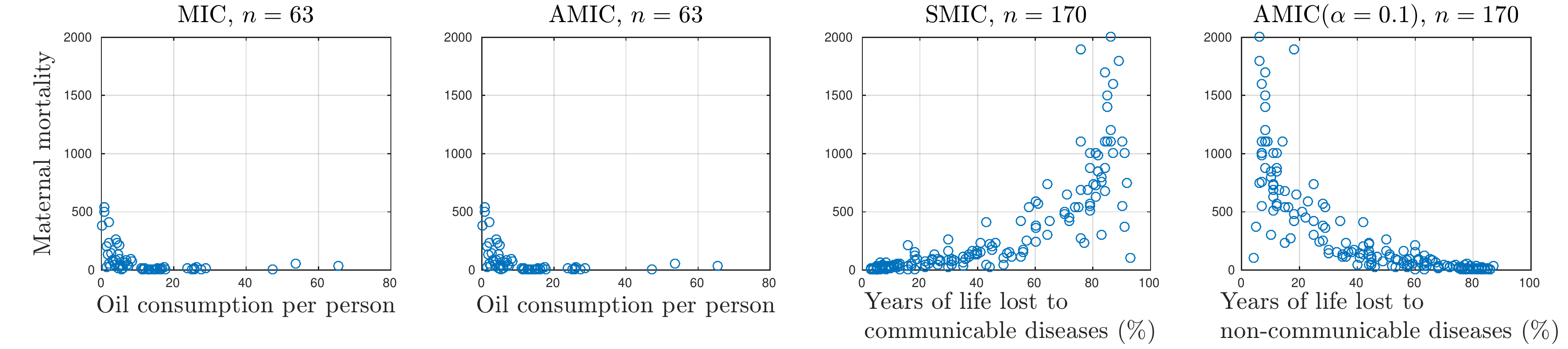}
\caption{Plot of the top-most dependent variable $X$ to $Y =$``\emph{Maternal Mortality}'' according different adjustments for MIC. MIC and AMIC are biased towards small $n$.
SMIC and AMIC$(\alpha = 0.1)$ select more reasonably either $X =$``\emph{Years of life lost to communicable diseases}'' or $X =$``\emph{Years of life lost to non-communicable diseases}''.}\label{fig:exmic}
\vspace{-1em} 
\end{figure*}
There seems to exist an inversely proportional relationship between $X$ and $Y$, possibly due to the common cause of overall economic development but it is difficult to argue in favor of the amount of oil/energy consumption per person as the most dependent variable to maternal mortality. We also identified the top variables according to SMIC and AMIC$(\alpha = 0.01)$ computed with 10,000 Monte Carlo permutations. More specifically, $\mbox{SMIC} = \frac{\mbox{MIC} - \mbox{EMIC}_0}{\mbox{SDMIC}_0}$, where SDMIC$_0$ is the unbiased estimator of the standard deviation of MIC permutations; and AMIC$(\alpha) = \mbox{MIC} - q_0(1-\alpha)$, where $q_0(1-\alpha)$ is the $\lceil (1- \alpha) \cdot S \rceil$-th MIC value from the sorted list of $S$ MIC permutations in ascending order (See Appendix~\ref{app:nullmic} for more details). The top variables according to SMIC and AMIC$(\alpha = 0.01)$ are instead variables related to communicable/non-communicable (infectious/non-infectious) diseases which is more intuitively related to mortality.  

Table~\ref{tbl:avgsize} shows the average sample size $n$ for the chosen top variables with different adjustments. The user of dependency measures should be aware of the bias of raw dependency estimators $\hat{\mathcal{D}}$ towards small $n$ and try to explore results from their adjusted versions $\mbox{S}\hat{\mathcal{D}}$ and $\mbox{A}\hat{\mathcal{D}}(\alpha)$ when ranking. Ultimately, the latter can be chosen to tune the bias towards smaller $n$ (big $\alpha$) or big $n$ (small $\alpha$).
\begin{table}
\centering 
\small
\caption{Average sample size $n$ for the top relationships in the WHO datasets. The raw estimator of a dependency measure $\hat{\mathcal{D}}$ favours relationships on small $n$. Instead, its standardized version S$\hat{\mathcal{D}}$ favours big $n$. With A$\hat{\mathcal{D}}(\alpha)$ it is possible to tune the bias towards small $n$ (big $\alpha$) or big $n$ (small $\alpha$).} \label{tbl:avgsize}
\begin{tabular}{lll}
\toprule
\emph{Measure} & $r^2$ & MIC \\
\toprule
$\hat{\mathcal{D}}$ & 114.6 (min) & 103.1 (min) \\
$\mbox{A}\hat{\mathcal{D}}$ & 115.1 & 106.9 \\
$\mbox{S}\hat{\mathcal{D}}$ & 133.7 (max) &  131.8 (max) \\
\hline
$\mbox{A}\hat{\mathcal{D}}(\alpha = 0.4)$ & 116.2 (min) & 111.7 (min)\\
$\mbox{A}\hat{\mathcal{D}}(\alpha = 0.05)$ & 121.1 & 119.4 \\
$\mbox{A}\hat{\mathcal{D}}(\alpha = 0.1)$ & 120.2 (max) & 117.4 (max)\\
\hline
\end{tabular}
\vspace{-1.5em} 
\end{table}

\subsection{Experiments with Gini gain in Random Forests.}
Splitting criteria are known to be biased towards variables induced on small $n$ or categorical with many categories. Standardized measures and $p$-values are the state-of-the-art strategy to solve this problem~\cite{Dobra01, Strobl2007Gini, Frank98, Romano2014}. However, we saw that standardized measures are unbiased in ranking only when the population value $\mathcal{D}(X,Y) = 0$, and the user might better tune the bias using the parameter $\alpha$. The optimal $\alpha$ can be found with cross-validation.

Here we use the expected value $E_0[\mbox{Gini}]$ and the variance $\mbox{Var}_0(\mbox{Gini})$ of Gini proposed in~\cite{Dobra01} to standardize Gini gain as per Definition~\ref{def:std} (See Appendix~\ref{app:nullgini} for more details). Moreover, we employ them to compute the adjusted Gini gain AGini$(\alpha)$ as follows:
\begin{restatable}{prop}{propginialpha}
The adjustment for ranking at level $\alpha \in (0,1]$ for \emph{Gini} gain is:
\[
\mbox{\emph{AGini}}(\mathcal{S}_n|X,Y)(\alpha) = \mbox{\emph{Gini}}(\mathcal{S}_n|X,Y) - \tilde{q}_0(1 - \alpha)
\]
where $\tilde{q}_0(1-\alpha)$ is an upper bound for the $(1-\alpha)$-quantile of \emph{Gini} gain equal to:
\[E[\text{\emph{Gini}}_0(\mathcal{S}_n|X,Y)]  + \sqrt{ \frac{1 - \alpha}{\alpha} \text{\emph{Var}}(\text{\emph{Gini}}_0(\mathcal{S}_n|X,Y))}.
\]
\end{restatable}
\noindent The proof of this upper bound is proposed in the supplement~\ref{app:nullgini}. 

We compare WEKA random forests with Gini, SGini, and AGini($\alpha$) as splitting criteria. To our knowledge this is the first time SGini and AGini($\alpha$) are tested in random forest. The forest is built on 1,000 trees taking care of sampling data with no replacement (50\% training set records for each tree) to not introduce further biases towards categorical variables with many categories~\cite{Strobl2007}. We employed 17 UCI datasets and 2 datasets with many categorical variables studied in~\cite{Altmann2010}. The latter datasets are related to biological classification problems and some of the variables can take as many categories as the number of amino acids at a given site in a viral protein: e.g.\ in the HIV dataset,
there exist variables which can take 21 possible values and induce splits of 21-cardinality in the trees. Table~\ref{tbl:rf} shows the AUC performance of random forest computed with 50 bootstrap 2-fold cross-validation using different splitting criteria.
All our adjustments improve on the AUC of the random forest built with Gini. We fixed $\alpha$ in AGini($\alpha$) to show that using a value of 0.05 or 0.1 on average increases the random forest's AUC: see Figure~\ref{fig:rfavg}.
\begin{figure}[h]
\centering
\includegraphics[scale=.8]{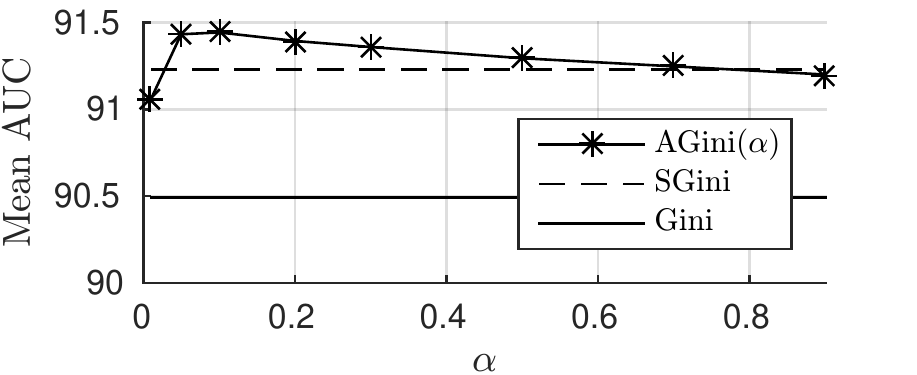}
\caption{AUC of random forest varying $\alpha$: with $\alpha = \{0.01,0.05\}$ it achieves the best results on average.}\label{fig:rfavg}
\end{figure}
Moreover, we also tuned $\alpha$ with cross-validation for the best performance, where small and big $\alpha$ correspond to penalization of variables with big and small number of categories respectively. Indeed, the performance of random forests with AGini$(\alpha)$ with $\alpha$ tuned is statistically better than the one built with Gini according to the 1-sided Wilcoxon signed rank test: $p$-value $=0.0086$. Although the observed effect size is small, it was consistent, and there is no extra computational effort.
We strongly believe that adjusted splitting criteria are beneficial given that i) they can be plugged in random forests where Gini is currently used to improve classification accuracy on data sets with categorical variables or with missing values,
ii) they exhibit the same computational complexity as the Gini, and iii) they are easy to implement, in particular much easier than the estimation of their confidence interval with a possibilistic loss function proposed recently~\cite{Serrurier2015}.
\begin{table*}
\scriptsize
\centering
\caption{
Random forest AUC using different splitting criteria. Either $(+)$, $(=)$, or $(-)$ means statistically greater, equal, or smaller according to the 1-sided paired $t$-test at level 0.05 than random forest AUC with Gini gain.
} \label{tbl:rf}
\begin{tabular}{p{1.8cm}HHR{1.7cm}R{1.1cm}R{1.9cm}rrR{1.6cm}R{1.6cm}R{1.6cm}}
\toprule
Dataset & perc missing & Min Cat & Variable with max number of categories & Number of classes & $m_{\textup{categorical}} + m_{\textup{continuous}} = m$ & $n$ & Gini  & SGini & AGini ($\alpha = 0.05$) & AGini($\alpha$) with $\alpha$ tuned \\
\toprule
Credit-g & 0.0\%  & 2 & 11 & 2 & $13+7=20$ & 1000 & 77.47 & \textbf{78.17} $(+)$ & 77.66 $(=)$ & 78.16 $(+)$ \\
australian & 0.0\%  & 2 & 14 & 2 & $8+6=14$ & 690 & 92.59 & 93.09 $(+)$ & 93.02 $(+)$ & \textbf{93.11} $(+)$ \\
bio-promoters & 0.0\%  & 4 & 4 & 2 & $57+0=57$ & 106 & 97.03 & 97.29 $(+)$ & 97.41 $(+)$ & \textbf{97.53} $(+)$ \\ 
flags & 0.0\%  & 2 & 14 & 8 & $26+2=28$ & 194 & 90.49 & 91.75 $(+)$ & 91.75 $(+)$ & \textbf{91.83} $(+)$ \\
kr-vs-kp & 0.0\%  & 2 & 3 & 2 & $36+0=36$ & 3196 & \textbf{99.86} & \textbf{99.86} $(=)$ & \textbf{99.86} $(=)$ & \textbf{99.86} $(=)$ \\ 
led7 & 0.0\%  & 2 & 2 & 10 & $7+0=7$ & 3200 & \textbf{94.18} & \textbf{94.18} $(=)$ & \textbf{94.18} $(=)$ & \textbf{94.18} $(=)$ \\ 
lymph & 0.0\%  & 2 & 8 & 4 & $15+3=18$ & 148 & 92.91 & \textbf{93.16} $(+)$ & 93.13 $(=)$ & 93.13 $(=)$ \\
mfeat-pixel & 0.0\%  & 5 & 7 & 10 & $240+0=240$ & 2000 & 99.58 & 99.63 $(+)$ & \textbf{99.64} $(+)$ & \textbf{99.64} $(+)$ \\
mito & 0.0\%  & 2 & 21 & 2 & $23+0=23$ & 175 & \textbf{79.32} & 79.28 $(=)$ & 79.26 $(=)$ & 79.10 $(=)$ \\
monks1 & 0.0\%  & 2 & 4 & 2 & $6+0=6$ & 556 & \textbf{99.96} & 99.85 $(-)$ & 97.38 $(-)$ & 99.78 $(-)$ \\
monks2 & 0.0\%  & 2 & 4 & 2 & $6+0=6$ & 601 & 64.86 & 70.89 $(+)$ & 77.83 $(+)$ & \textbf{80.72} $(+)$ \\
monks3 & 0.0\%  & 2 & 4 & 2 & $6+0=6$ & 554 & 98.73 & \textbf{98.74} $(=)$ & \textbf{98.74} $(=)$ & 98.73 $(=)$ \\
solar-flare & 0.0\%  & 2 & 6 & 6 & $11+0=11$ & 323 & 89.17 & \textbf{89.23} $(+)$ & 89.22 $(=)$ & \textbf{89.23} $(+)$ \\
splice & 0.0\%  & 4 & 6 & 3 & $60+0=60$ & 3190 & \textbf{99.52} & \textbf{99.52} $(=)$ & \textbf{ 99.52} $(=)$ & \textbf{99.52} $(=)$ \\
steel & 0.0\%  & 2 & 2 & 2 & $6+27=33$ & 1941 & \textbf{99.94} & 99.93 $(-)$ & 99.93 $(-)$ & \textbf{99.94} $(-)$ \\
tae & 0.0\%  & 2 & 2 & 3 & $2+3=5$ & 151 & 72.25 & 72.33 $(+)$ & 73.23 $(+)$ & \textbf{73.65} $(+)$ \\
tic-tac-toe & 0.0\%  & 3 & 3 & 2 & $9+0=9$ & 958 & 97.83 & 97.93 $(+)$ & \textbf{97.95} $(+)$ & 97.94 $(+)$ \\
\hline
c-to-u & 0.0\%  & 1 & 5 & 2 & $42+3=45$ & 2694 & \textbf{89.74} & 89.42 $(-)$ & 89.28 $(-)$ & 89.61 $(-)$ \\
HIV & 0.0\%  & 1 & 21 & 2 & $1030+0=1030$ & 355 & 84.08 & 89.27 $(+)$ & \textbf{89.58} $(+)$ & \textbf{ 89.58} $(+)$ \\
\hline
\multicolumn{8}{r}{$p$-value for the 1-tailed Wilcoxon signed rank test against random forest with Gini} & 
 0.0114 & 0.0295 & 0.0086 \\
 \hline
\end{tabular}
\vspace*{-1.5em} 
\end{table*}

\section{Conclusion}

In this paper we discussed how to adjust dependency measure estimates between two variables $X$ and $Y$ using the null hypothesis of their independence. This is particularly important to achieve \emph{interpretable quantification} of the amount of dependency. For this task, we proposed the quantification adjusted measures $\mbox{A}r^2$ and AMIC. However, quantification adjustment is not enough to achieve \emph{accurate ranking} of dependencies. In particular, it is very difficult to achieve ranking unbiasedness. In this task, the user should explore the possible rankings obtained with standardized and ranking adjusted measures, varying the parameter $\alpha$. We demonstrated that our S$r^2$, A$r^2(\alpha)$, SMIC, and AMIC$(\alpha)$ can be used to obtain more meaningful rankings, and that AGini$(\alpha)$ yields higher accuracy in random forests. The code for our measures, experiments, and supplementary material have been made available online\footnote{\url{https://sites.google.com/site/adjdep/}}. 
\vspace{-1em} 

\subsection*{Acknowledgments:}
Supported by AWS in Education Grant Award and ARC FT110100112.
\vspace{-1em} 

\bibliographystyle{IEEEtran}
\bibliography{mandb_SDM16}

\clearpage

\appendix

\twocolumn[
\centering
{\bfseries \huge Supplementary Material}
\vspace{2em}
]
\section{Dependency Measure Estimators} \label{app:def}

Here we formally define the dependency measure estimator of Gini gain between two categorical variables $X$ and $Y$ and the Maximal Information Coefficient (MIC)~\cite{Reshef2011} on a sample $\mathcal{S}_n$.

\noindent
{\bfseries Gini gain and mutual information (a.k.a. information gain)}

Let $X$ be a categorical variable with $r$ possible values, then $n_i^X$ with $i=1,\dots,r$ is the count of records with value $i$ for $X$ in the sample $\mathcal{S}_n$. Similarly, let $Y$ be a categorical variable with $c$ possible values, $n_j^Y$ with $j=1,\dots,c$ is the count of records with value $j$ for $Y$. Finally, the count of records for the pairs that associate the values $i$ for $X$ and $j$ for $Y$ is denoted as $n_{ij}$. Gini gain are estimated on the empirical probabilities $\frac{n_{ij}}{n}$, $\frac{n_i^X}{n}$, and $\frac{n_j^Y}{n}$ which can be stored in a contingency table:
\begin{figure}[h]
\centering
\begin{tabular}{c|c|ccccc|}
\multicolumn{2}{c}{} & \multicolumn{5}{c}{  $Y$     }\\
\cline{3-7}
\multicolumn{2}{c|}{ } & $n_1^Y$ & $\cdots$ & $n_j^Y$ & $\cdots$ & $n_c^Y$ \\
\cline{2-7}
\multirow{2}{*}{     }  & $n_1^X$ &
$n_{11}$ & $\cdots$ & $\cdot$ & $\cdots$ & $n_{1c}$ \\
& $\vdots$ &
$\vdots$ &  & $\vdots$ &  & $\vdots$ \\
 $X$  & $n_i^X$ &
$\cdot$ &  & $n_{ij}$ & & $\cdot$ \\
& $\vdots$ &
$\vdots$ & & $\vdots$ & & $\vdots$ \\
& $n_r^X$ &
$n_{r1}$ & $\cdots$ & $\cdot$ & $\cdots$ & $n_{rc}$ \\
\cline{2-7}
\multicolumn{7}{c}{      }\\
\end{tabular}
\caption{$r \times c$ contingency table that stores the bivariate frequency distribution of $X$ and $Y$ for the sample $\mathcal{S}_n$.}
\label{fig:contingency}
\end{figure}

\noindent Gini gain is defined as:
\begin{equation} \label{eq:gini}
\mbox{Gini}(\mathcal{S}_n|X,Y) \triangleq 1 - \sum_{j=1}^c \Big(\frac{n_j^Y}{n}\Big)^2 - \sum_{i=1}^{r} \frac{n_i^X}{n} \Big( 1 - \sum_{i=j}^{c}\Big(\frac{n_{ij}}{n_i^X}\Big)^2\Big)
\end{equation}
Instead, the mutual information (MI) between $X$ and $Y$ is defined as: 
\begin{equation}
\mbox{MI}(\mathcal{S}_n|X,Y) \triangleq \sum_{i=1}^r \sum_{j=1}^c \frac{n_{ij}}{n} \log_2{ \frac{n_{ij} \cdot n }{n_i^X n_j^Y} }
\end{equation}

\noindent
{\bfseries Maximal Information Coefficient (MIC)}

Given a sample $\mathcal{S}_n$ from the continuous variables $X$ and $Y$, MIC is the maximal normalized MI computed across all the possible $r \times c$ grids to estimate the bivariate frequency distribution of $X$ and $Y$. Each $r \times c$ discretizes the scatter plot of $X$ and $Y$ in $r\cdot c$ bins to compute their frequency distribution:
\begin{equation}
\mbox{MIC}(\mathcal{S}_n|X,Y) \triangleq \max_{r \times c \mbox{ \tiny grids with } r \cdot c \leq n^{a}} \frac{\mbox{MI}(\mathcal{S}_n|X,Y)}{ \log_2{ \min{ \{ r,c\} }} } 
\end{equation} 
where $a$ is a parameter often set to $0.6$~\cite{Reshef2011}.

\subsection{Distribution of MIC under the null.} \label{app:nullmic}

The distribution of $\mbox{MIC}_0(\mathcal{S}_n|X,Y)$ can be computed using $s=1,\dots,S$ Monte Carlo instances $\mbox{MIC}^{(s)}_0$ of MIC computed on the sample $\mathcal{S}^0_n = \{ (x_{\sigma_x(k)}, y_{\sigma_y(k)}) \}$ obtained by permutations of the sample $\mathcal{S}_n$: $\sigma_x(k)$ and $\sigma_y(k)$ are the permuted indexes of the points $x_k$ and $y_k$ respectively. The expected value of MIC under the null can be estimated with:
\begin{equation}
\mbox{EMIC}_0 = \frac{1}{S}\sum_{s=1}^S \mbox{MIC}^{(s)}_0
\end{equation}
The standard deviation of MIC under the null can be estimated with:
\begin{equation}
\mbox{SDMIC}_0 = \sqrt{\frac{1}{S -1} \sum_{s=1}^S (\mbox{MIC}^{(s)}_0 - \mbox{EMIC}_0)^2}
\end{equation}

\subsection{Distribution of Gini gain under the null.} \label{app:nullgini}

The analytical distribution of Gini gain in Eq.~\eqref{eq:gini} is difficult to compute. Nonetheless, it is possible to compute its expected value and variance. According~\cite{Dobra01} the expected value of Gini gain under the null using the multinomial model is:
\begin{equation}
E[\mbox{Gini}_0(\mathcal{S}_n|X,Y)] = \frac{r -1}{n} \Big( 1 - \sum_{j=1}^c\Big(\frac{n_j^Y}{n}\Big)^2\Big)
\end{equation}
and the variance $\mbox{Var}[\mbox{Gini}_0(\mathcal{S}_n|X,Y))$ is:
\begin{align}
\frac{1}{n^2} \Bigg[ (r -1) \Big( 2  \sum_{j=1}^c\Big(\frac{n_j^Y}{n}\Big)^2 +2 \Big(\sum_{j=1}^c\Big(\frac{n_j^Y}{n}\Big)^2\Big)^2 - 4 \sum_{j=1}^c\Big(\frac{n_j^Y}{n}\Big)^3 \Big)\\\nonumber
+ \Big( \sum_{i=1}^r \frac{1}{n_i^X} - 2 \frac{r}{n} + \frac{1}{n}\Big) \times \\ \nonumber
\Big( -2\sum_{j=1}^c\Big(\frac{n_j^Y}{n}\Big)^2 -6 \Big(\sum_{j=1}^c\Big(\frac{n_j^Y}{n}\Big)^2\Big)^2 + 8 \sum_{j=1}^c\Big(\frac{n_j^Y}{n}\Big)^3\Big)\Bigg]
\end{align}
We can compute the $(1-\alpha)$-quantile of Gini gain under the null using its expected value and variance:

\propginialpha*
\begin{proof}
Let $\mu$ and $\sigma$ be the expected value and standard deviation respectively. We apply the Cantelli's inequality to find an upper bound for $q_0(1-\alpha)$: $P (\mbox{Gini} \leq \mu +  \lambda \sigma ) \geq \frac{\lambda^2}{1 + \lambda^2}$ for $\lambda \geq 0$. If we set $\frac{\lambda^2}{1 + \lambda^2} = \alpha$ then $P( \mbox{Gini} \leq \mu  + \sqrt{ \frac{1 - \alpha}{\alpha} } \sigma ) \geq \alpha$. This implies $q_0(1-\alpha) \leq \mu  + \sqrt{ \frac{1 - \alpha}{\alpha}} \sigma $. \hfill $\square$
\end{proof}

\end{document}